\documentclass{article}
\pdfoutput=1

\usepackage[final, nonatbib]{neurips_2022}

\usepackage[utf8]{inputenc} 
\usepackage[T1]{fontenc}    
\usepackage{hyperref}       
\usepackage{url}            
\usepackage{booktabs}       
\usepackage{amsfonts}       
\usepackage{nicefrac}       
\usepackage{microtype}      
\usepackage{xcolor}         
\usepackage{amsthm}
\usepackage{amsmath}
\usepackage{multirow}
\usepackage{graphicx}
\usepackage{caption}
\usepackage{subcaption}

\newtheorem{thm}{Theorem}[section]
\newtheorem{lem}{Lemma}[section]
\newtheorem{theorem}{Theorem}
\newtheorem{lemma}{Lemma}

\title{COIN: Co-Cluster Infomax for Bipartite Graphs}

\author{%
  Baoyu Jing$^\dagger$, Yuchen Yan$^\dagger$, Yada Zhu$^\ddagger$ and Hanghang Tong$^\dagger$ \\
  $^\dagger$University of Illinois at Urbana-Champaign\\
  $^\ddagger$IBM Research\\
  \texttt{\{baoyuj2, yucheny5, htong\}@illinois.edu},\, \texttt{yzhu@us.ibm.com} \\
}

\begin{document}

\maketitle

\begin{abstract}
\pdfoutput=1

Bipartite graphs are powerful data structures to model interactions between two types of nodes, which have been used in a variety of applications, such as recommender systems, information retrieval, and drug discovery.
A fundamental challenge for bipartite graphs is how to learn informative node embeddings.
Despite the success of recent self-supervised learning methods on bipartite graphs, their objectives are discriminating instance-wise positive and negative node pairs, which could contain cluster-level errors.
In this paper, we introduce a novel \underline{co}-cluster \underline{in}fomax (COIN) framework, which captures the cluster-level information by maximizing the mutual information of co-clusters.
Different from previous infomax methods which estimate mutual information by neural networks, COIN could easily calculate mutual information.
Besides, COIN is an end-to-end co-clustering method which can be trained jointly with other objective functions and optimized via back-propagation.
Furthermore, we also provide theoretical analysis for COIN.
We theoretically prove that COIN is able to effectively increase the mutual information of node embeddings and COIN is upper-bounded by the prior distributions of nodes.
We extensively evaluate the proposed COIN framework on various benchmark datasets and tasks to demonstrate the effectiveness of COIN.
\end{abstract}

\section{Introduction}\label{sec:introduction}
\pdfoutput=1

Graphs have attracted plenty of attention in recent years \cite{ velickovic2019deep,grover2016node2vec, he2017neural, jing2021multiplex, zhu2021graph, du2021new, jing2021network, wu2021self, yan2021dynamic, yan2022dissecting}. 
The bipartite graph is a powerful representation formalism to model interactions between two types of nodes, which has been used in a variety of real-world applications.
For example, in recommender systems \cite{aggarwal2016recommender, wang2021graph},  users, items and their interactions (e.g. buy) is a natural bipartite graph;
in information retrieval \cite{beeferman2000agglomerative, he2016birank}, clickthrough between queries and webpages can be conveniently modeled by a bipartite graph;
in drug discovery \cite{pavlopoulos2018bipartite, yamanishi2008prediction}, chemical interactions (e.g. nuclear receptor) between drugs and target proteins can also be represented by a bipartite graph. 

One fundamental challenge for bipartite graphs is how to extract informative node embeddings, such that they can be easily used for downstream tasks (e.g. link prediction).
In recent years, self-supervised learning has become a popular paradigm to learn node embeddings without human labels \cite{liu2021graph, wu2021self, ding2022data, DBLP:conf/cikm/ZhouZF0H22, zheng2022contrastive}.  
Despite their success in extracting high-quality node embeddings, which have great performance on downstream tasks, most of them are designed for homogeneous graphs \cite{hassani2020contrastive, perozzi2014deepwalk, velickovic2019deep, you2020graph, zhu2021graph, feng2022adversarial, li2022graph, wang2022augmentation} and heterogeneous graphs \cite{park2020unsupervised, jing2021hdmi, DBLP:conf/cikm/FuXLTH20, jing2022x, du2021new}.
Thus they are sub-optimal to bipartite graphs \cite{cao2021bipartite, gao2018bine}.
Several methods have been specifically proposed for bipartite graphs.
For example, BiNE \cite{gao2018bine} learns embeddings by maximizing the similarity of neighbors sampled by random walks; NeuMF \cite{he2017neural} and GC-MC \cite{berg2017graph} train neural networks by reconstructing the edges; BiGI \cite{cao2021bipartite} and EGLN \cite{yang2021enhanced} further improve the quality of node embeddings by maximizing the mutual information between local and global representations.

Although promising results have been achieved by the aforementioned methods, they learn node embeddings by discriminating instance-wise positive pairs (e.g. local neighbors) and negative pairs (e.g. randomly sampled unconnected node pairs), but ignore the cluster-level information.
Such a practice restricts the quality of the learned embeddings since instance-wise negative pairs might contain cluster-level errors \cite{caron2020unsupervised, li2020prototypical, jing2022x}.
Clusters naturally exist in real-world bipartite graphs, such as the categories of items in a user-item graph and the topics of documents in a document-keywork graph.
Regardless of the cluster information, one might wrongly pair two nodes within the same cluster as a negative pair, which will lead to errors in downstream tasks \cite{jing2022x, li2020prototypical}.

Due to the duality between two types of nodes in a bipartite graph, co-clustering two types of nodes simultaneously usually yields better results than traditional clustering algorithms \cite{dhillon2003information, xu2019deep}.
In this paper, we introduce a novel \underline{co}-cluster \underline{in}fomax (COIN) framework to incorporate cluster-level information into the node embeddings of bipartite graphs.
Given a bipartite graph, COIN first uses neural networks to cluster two types of nodes into co-clusters and then maximizes the mutual information of the co-clusters.
There are two advantages of the proposed COIN framework.
Firstly, COIN directly \emph{calculates} the mutual information of the co-clusters, rather than \emph{estimating} mutual information via neural networks \cite{belghazi2018mutual} in prior works \cite{cao2021bipartite, velickovic2019deep}, which could be unreliable for complex distributions in practice \cite{song2019understanding}. 
Secondly, COIN is an end-to-end co-clustering method, which is differentiable and can be trained jointly with other objective functions.

We further present the theoretical analysis and empirical evaluation of COIN.
In theoretical analysis, we prove that (1) maximizing the mutual information of co-clusters will increase the mutual information of the node embeddings, and (2) the mutual information of co-clusters is upper-bounded by the prior distribution assumed over the bipartite graph.
In empirical evaluation, we extensively evaluate the proposed COIN on various public real-world benchmark datasets and downstream tasks to demonstrate the effectiveness of COIN.

The contributions of this paper are summarized as follows:
\begin{itemize}
    \item We introduce a novel framework COIN for self-supervised learning on bipartite graphs, which incorporates cluster information by maximizing the mutual information of co-clusters.
    \item We theoretically prove that COIN maximizes the mutual information between the embeddings of two types of nodes, and COIN is upper-bounded by the prior distribution.
    \item We extensively evaluate COIN on various benchmark datasets and downstream tasks, including link prediction, recommendation, and clustering, to demonstrate its effectiveness. 
\end{itemize}

\section{Preliminary}\label{sec:preliminary}
\pdfoutput=1


\paragraph{Self-Supervised Learning for Bipartite Graphs.}
We denote a bipartite graph as $G=(U, V, E)$, where $U$ and $V$ are two disjoint sets of nodes, and $E\subseteq U\times V$ is the set of edges.
Our goal is to train a graph encoder $(\mathbf{U}, \mathbf{V})=\mathcal{E}(G)$ to extract informative node embeddings $\mathbf{U}\in\mathbb{R}^{|U|\times d}$, $\mathbf{V}\in\mathbb{R}^{|V|\times d}$ from $G$, where $d$ is the size of hidden dimension. 
When there is no ambiguity, we also use $U$, $V$ as the random variables for nodes, $\mathbf{U}$, $\mathbf{V}$ as the random variables for node embeddings. 
Correspondingly, $u$, $v$ and $\mathbf{u}$, $\mathbf{v}$ are used as the indices for $U$, $V$ and $\mathbf{U}$, $\mathbf{V}$.

\paragraph{Co-Clustering.}
Given a biparitite graph ${G}=({U}, {V}, {E})$, the \emph{soft} co-clustering aims to map nodes $U$ and $V$ into $N_K\ll|U|$ and $N_L\ll|V|$ clusters via the function $\phi = (\phi_U, \phi_V)$, where $\phi_U(u)\in\mathbb{R}^{N_K}$ and $\phi_V(v)\in\mathbb{R}^{N_L}$ produce the conditional probabilities of cluster assignments $p(k|u)$, $p(l|v)$ for nodes $u$, $v$.
Here $k\in[1, \cdots, N_K]$ and $l\in[1, \cdots, N_L]$ are the indices of the clusters.
Furthermore, we use $K$ and $L$ to denote random variables of co-clusters. 


\section{Methodology}\label{sec:method}
\pdfoutput=1


\begin{figure}
    \centering
    \includegraphics[width=0.9\textwidth]{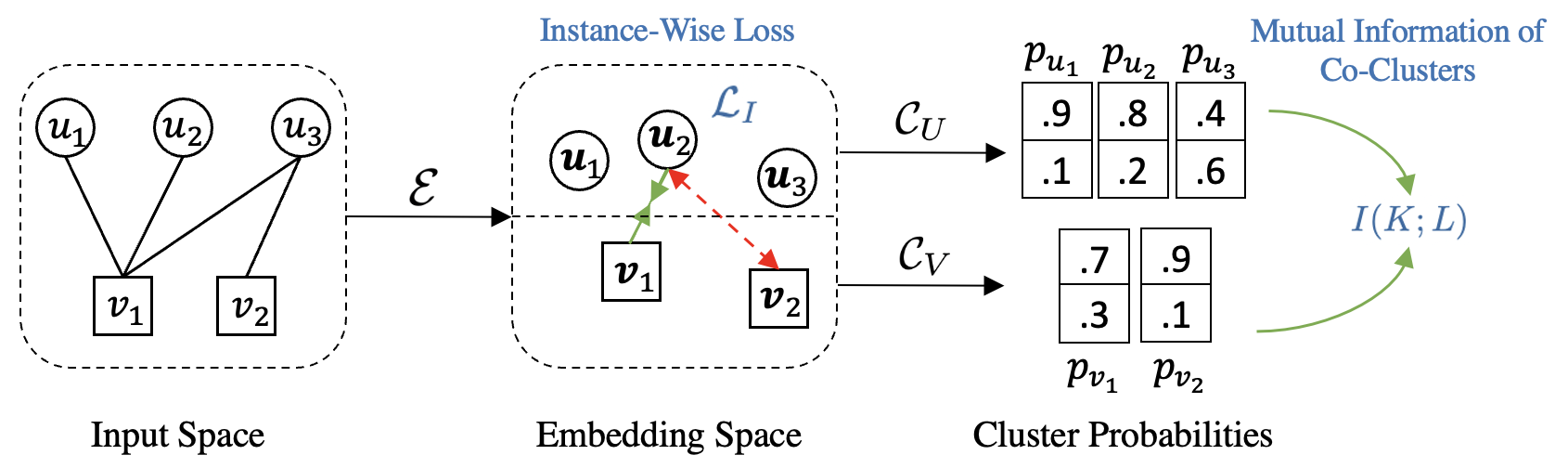}
    \caption{Overview of COIN. The proposed COIN is comprised of a co-clustering function $\phi=(\phi_U, \phi_V)$, instance-wise objective $\mathcal{L}_I$ and mutual information of co-clusters $I(K;L)$. $\phi_U$ and $\phi_V$ (1) share the same graph encoder $\mathcal{E}$, which maps nodes $u$, $v$ into embeddings $\mathbf{u}$, $\mathbf{v}$, and (2) have separate cluster networks $\mathcal{C}_U$ and $\mathcal{C}_V$, which produce probabilities of different clusters $p_u$ and $p_v$ for nodes $u$ and $v$.
    $\mathcal{L}_I$ pulls positive node pairs closer (green solid arrow) and pushes negative node pairs far apart (red dashed arrow). $I(K;L)$ is calculated based on $p_u$ and $p_v$.}
    \label{fig:overview}
\end{figure}

\subsection{Co-Cluster Infomax}
Prior deep infomax based frameworks \cite{hassani2020contrastive, velickovic2019deep, jing2021hdmi, park2020unsupervised} maximize the mutual information of local and global representations, which ignore the co-cluster information and the neural mutual information estimators could be unreliable in practice.
Different from these works, COIN directly calculates and then maximizes the mutual information of co-clusters $I(K;L)$ to capture cluster-level information.
An overview of COIN is presented in Figure \ref{fig:overview}.

\paragraph{Mutual Information of Co-Clusters.}
To calculate the mutual information of co-clusters $I(K;L)$, we need to obtain the joint distribution $p(k, l)$ and marginal distributions $p(k)$ and $p(l)$.
To calculate $p(k, l)$, we decompose $p(k, l)$ by two other easy-to-obtain distributions $p(u, v)$ and $p(k, l|u, v)$: $p(k, l)=\sum_{u, v}p(k, l|u, v)p(u, v)$.
Here, $p(u, v)$ is the prior assumption over the joint distribution of nodes $u$, $v$, which preserves the structure of $G$.
For simplicity, we define the prior $p(u, v)=\frac{1}{Z}$ if $(u, v)$ is connected, where $Z$ is a normalization parameter, $p(u, v)=0$ otherwise.

With the prior distribution $p(u, v)$, the next step is to obtain the conditional joint distribution $p(k, l|u, v)$.
Given a bipartite graph ${G} = ({U}, {V}, {E})$, we use a co-clustering algorithm $\phi=(\phi_U, \phi_V)$ to obtain cluster distributions for each node.
We define $\phi_U$ and $\phi_V$ via neural networks, which are illustrated in Figure \ref{fig:overview}.
They are comprised of two components: 
(1) a shared graph encoder to obtain node embeddings $(\mathbf{U}, \mathbf{V})=\mathcal{E}(G)$, and
(2) separate cluster networks to obtain conditional probabilities of clusters $\mathcal{C}_U(\mathbf{u})\in\mathbb{R}^{N_K}$ and $\mathcal{C}_V(\mathbf{v})\in\mathbb{R}^{N_L}$, where $\mathbf{u}\in\mathbf{U}$ and $\mathbf{v}\in\mathbf{V}$ are embeddings of nodes $u$ and $v$.
Since neural encoders are usually treated as deterministic and injective functions in practice \cite{vincent2008extracting}, we have $p(k, l|u, v)=p(k, l|\mathbf{u}, \mathbf{v})$, where $\mathbf{u}$ and $\mathbf{v}$ are obtained by $\mathcal{E}$.
Furthermore, since cluster networks $\mathcal{C}_U$ and $\mathcal{C}_V$ have separate sets of parameters, therefore, it is natural to assume $p(k, l|\mathbf{u}, \mathbf{v})=p(k|\mathbf{u})p(l|\mathbf{v})$. 
As a result, $p(k, l|u, v)=p(k|\mathbf{u})p(l|\mathbf{v})$.
Finally, combining $p(k|\mathbf{u})p(l|\mathbf{v})$ with the assumed prior distribution $p(u, v)$, we have 
\begin{equation}\label{eq:pkl}
    p(k, l) = \sum_{u, v}p(k|\mathbf{u})p(l|\mathbf{v})p(u, v)
\end{equation}
where $\mathbf{u}$, $\mathbf{v}$ are obtained by $\mathcal{E}$, and $p(k|\mathbf{u})$, $p(l|\mathbf{v})$ are obtained by $\mathcal{C}_U$, $\mathcal{C}_V$.

Given the joint distribution of co-clusters $p(k, l)$, we can easily calculate the marginal distributions $p(k)=\sum_{l}p(k, l)$ and $p(l)=\sum_{k}p(k, l)$.
According to the definition of mutual information, we can directly calculate $I(K;L)$ by $I(K; L) = \sum_{k, l}p(k, l)\log\frac{p(k, l)}{p(k)p(l)}$.
Since $p(k, l)$, $p(k)$ and $p(l)$ are calculated by neural networks, which are fully differentiable, we can directly maximize $I(K;L)$. 

\paragraph{Instance-Wise Objective Function.}
A clustering objective function alone is incapable of achieving the optimal performance \cite{li2020prototypical, li2021contrastive, jing2022x}, since it is unable to capture local information such as the connectivity between two nodes $u$, $v$.
Therefore, an instance-wise objective function is usually used together with the clustering objective function. 
In this paper, we use an objective function, shown in Equation \eqref{eq:instance_wise_object}, based on the popular InfoNCE \cite{oord2018representation} as the instance-wise objective function, which pulls connected $u$, $v$ together and pushes away unconnected $u$, $v'$ in the embedding space, as illustrated in the middle part of Figure \ref{fig:overview}.
\begin{equation}\label{eq:instance_wise_object}
    \max\mathcal{L}_I = \max\sum_{\mathbf{u}, \mathbf{v}}\frac{e^{\mathcal{S}(\mathbf{u}, \mathbf{v})}}{e^{\mathcal{S}(\mathbf{u}, \mathbf{v})} + e^{\mathcal{S}(\mathbf{u}, \mathbf{v}')}}
\end{equation}
where $\mathcal{S}$ is the similarity measuring function, and $\mathbf{v}'$ is the embedding of the negative sample $v'$, which is randomly sampled from $V'\subseteq V$, such that $(u, v')$ is not connected: $(u, v')\notin E$ for $\forall v'\in V'$.
Most of the recent InfoNCE based studies \cite{li2020prototypical, you2020graph, zhu2021graph, feng2022ariel, zheng2021tackling} use the normalized temperature-scaled cross entropy loss (NT-Xent) variant of InfoNCE. 
They instantiate $\mathcal{S}$ with unlearnable functions (e.g. cosine similarity and dot product) and normalize the similarity scores with a temperature parameter $\tau$.
However, unlearnable functions might not be the optimal choice for complex manifolds in the embedding space \cite{nickel2017poincare}, and tuning the hyper-parameter $\tau$ could be time-consuming.
To tackle with this issue, we use a multi-layer perceptron (MLP) to instantiate $S$.

\paragraph{Overall Objective Function.}
The overall objective function is a linear combination of the mutual information of co-clusters $I(K;L)$ and the instance-wise objective function $\mathcal{L}_I$:
\begin{equation}
    \max\mathcal{L} = \max \lambda I(K;L) + \mathcal{L}_{I} 
\end{equation}
where $\lambda$ is the co-efficient for $I(K;L)$.

\subsection{Theoretical Analysis}
We provide theoretical analysis for the proposed COIN framework with respect to the mutual information $I(K;L)$ of the co-clusters $K$ and $L$.

We study the relationship between the mutual information of co-clusters $I(K;L)$ and the mutual information of embeddings $I(\mathbf{U};\mathbf{V})$.
Based on Lemma \ref{lemma}, we first introduce Theorem \ref{theorem:variational_bound}, where we prove that $I(K;L)$ is a variational lower-bound for $I(\mathbf{U};\mathbf{V})$.
This theorem implies that maximizing $I(K;L)$ will increase $I(\mathbf{U};\mathbf{V})$.

\begin{lemma}\label{lemma}
For COIN, the following inequality holds:
\begin{equation}
    \log p(k)p(l) \geq \sum_{\mathbf{u}, \mathbf{v}}p(\mathbf{u}, \mathbf{v}|k, l)\log\frac{p(\mathbf{u}, k)p(\mathbf{v}, l)}{p(\mathbf{u}, \mathbf{v}|k, l)}
\end{equation}
where $k\in\{1, \cdots, N_K\}$, $l\in\{1, \cdots, N_L\}$, $\mathbf{u}\in\mathbf{U}$, $\mathbf{v}\in\mathbf{V}$.
\end{lemma}

\begin{proof}
Please refer to Lemma \ref{lemma_appendix} in Appendix \ref{appendix:theory}.
\end{proof}

\begin{theorem}[Variational Bound]\label{theorem:variational_bound}
The mutual information $I(\mathbf{U};\mathbf{V})$ of embeddings $\mathbf{U}$ and $\mathbf{V}$ is lower-bounded by the mutual information of co-clusters $I(K;L)$:
\begin{equation}\label{eq:theorem_bound}
    I(K;L)\leq I(\mathbf{U};\mathbf{V})
\end{equation}
\end{theorem}

\begin{proof}
According to the definition of mutual information and Lemma \ref{lemma}, we have:
\begin{equation}\label{eq:thorem_1_1}
    I(K; L)= \sum_{k, l}p(k, l)\frac{\log p(k, l)}{\log p(k)p(l)} \leq \sum_{k, l}p(k, l)\sum_{\mathbf{u}, \mathbf{v}}p(\mathbf{u}, \mathbf{v}|k, l)\log \frac{p(k, l)p(\mathbf{u}, \mathbf{v}|k, l)}{p(\mathbf{u}, k)p(\mathbf{v}, l)} = R
\end{equation}
By merging $p(k, l)$ with $p(\mathbf{u}, \mathbf{v}|k, l)$, we have:
\begin{equation}
    R = \sum_{\mathbf{u}, \mathbf{v}, k, l}p(\mathbf{u}, \mathbf{v}, k, l)\log \frac{p(\mathbf{u}, \mathbf{v}, k, l)}{p(\mathbf{u}, k)p(\mathbf{v}, l)} = \sum_{\mathbf{u}, \mathbf{v}, k, l}p(\mathbf{u}, \mathbf{v}, k, l)\log \frac{p(k, l|\mathbf{u}, \mathbf{v})p(\mathbf{u}, \mathbf{v})}{p(\mathbf{u})p(k|\mathbf{u})p(\mathbf{v})p(l|\mathbf{v})}
\end{equation}
Since cluster networks $\mathcal{C}_U$ and $\mathcal{C}_V$ have separate sets of parameters and inputs, therefore, it is natural to have $p(k, l|\mathbf{u}, \mathbf{v})=p(k|\mathbf{u})p(l|\mathbf{v})$.
As a result, we have:
\begin{equation}
    R = \sum_{\mathbf{u}, \mathbf{v}, k, l}p(\mathbf{u}, \mathbf{v}, k, l)\log \frac{p(\mathbf{u}, \mathbf{v})}{p(\mathbf{u})p(\mathbf{v})} = \sum_{\mathbf{u}, \mathbf{v}}p(\mathbf{u}, \mathbf{v})\log \frac{p(\mathbf{u}, \mathbf{v})}{p(\mathbf{u})p(\mathbf{v})} = I(\mathbf{U};\mathbf{V})
\end{equation}
Take the above result back to Inequality \eqref{eq:thorem_1_1}, and we will obtain Inequality \eqref{eq:theorem_bound}. 
\end{proof}

We also study the relationship between mutual information of co-clusters $I(K;L)$ and mutual information of nodes $I(U;V)$, where $I(U;V)$ is calculated based on the prior joint distribution $p(U, V)$ \emph{assumed} over nodes $U$ and $V$.
In Theorem \ref{theorem:diff_mi}, we derive the different between $I(K;L)$ and $I(U;V)$. 
Note that different from the theory in \cite{dhillon2003information}, which considers \emph{hard} clustering, we consider \emph{soft} clustering, and thus the proofs and conclusions are different.
From Theorem \ref{theorem:diff_mi}, we can obtain that $I(U;V)-I(K;L)=D_{KL}(p||q)\geq 0$, which implies that the learning ability of COIN is upper-bounded by the prior distribution $p(U,V)$.

\begin{theorem}[Mutual Information Difference]\label{theorem:diff_mi}
Given the prior joint distribution $p(U, V)$ assumed on $U$, $V$, and a soft co-clustering function $\phi=(\phi_U, \phi_V)$, where $\phi_U$ and $\phi_V$ are deterministic functions mapping nodes $U$,$V$ to their cluster distributions $p(K|U)$ and $p(L|V)$, we have:
\begin{equation}
    I(U, V) - I(K;L) = D_{KL}(p(K,L,U, V)||q(K,L,U, V))
\end{equation}
where $D_{KL}$ denotes the KL-divergence, and $q(K,L,U,V)=p(K,L)p(U|K)p(V|L)$.
\end{theorem}

\begin{proof}
Please refer to Theorem \ref{theorem_appendix:diff_mi} in Appendix \ref{appendix:theory}.
\end{proof}

\section{Experiments}\label{sec:experiments}
\pdfoutput=1

In this section, we extensively evaluate the effectiveness of COIN on a variety of benchmark datasets and downstream tasks, and provide empirical evaluation results for COIN.

\subsection{Experimental Setup}\label{sec:experimental_setup}
\paragraph{Datasets.}
We evaluate the proposed COIN on six public benchmark datasets with three different downstream tasks.
The descriptions of the datasets are presented in Table \ref{tab:data}.
Wikipedia contains the edit relationship between authors and pages, which is used for link prediction.
We use the data processed by \cite{cao2021bipartite}, which has two different splits: 50\% and 40\% for training.
ML-100K and ML-10M \cite{harper2015movielens} contain the ratings of users for movies, which are used for top-K recommendation.
WebKB, Wisconsin, and IMDB are document-keyword interaction datasets, which are used for co-clustering \cite{xu2019deep}. 
Further descriptions and the links of the datasets are provided in Appendix \ref{appendix:data}.

\paragraph{Evaluation Metrics.}
For link prediction, we use Area Under the ROC and Precision-Recall Curves, i.e. AUC-ROC and AUC-PR.
For top-K recommendation, F1 score, Normalized Discounted Cumulative Gain (NDCG), Mean Average Precision (MAP), and Mean Reciprocal Rank (MRR) are used.
For co-clustering, we use Normalized Mutual Information (NMI).

\begin{table}[t]
    \centering
    \scriptsize
    \caption{Descriptions of datasets}
    \begin{tabular}{cccccccc}
        \hline
        Dataset & Task & Evaluation & $|U|$ & $|V|$ & $|E|$ & Density & \# Class\\
        \hline
        Wikipedia & Link Prediction & AUC-ROC, AUC-PR & 15,000 & 3,214 & 64,095 & 0.1\% & -\\
        \hline
        ML-100K & Top-K Recommendation & F1, NDCG, MAP, MRR & 943 & 1,682 & 100,000 & 6.3\% & -\\
        ML-10M & Top-K Recommendation & F1, NDCG, MAP, MRR & 69,878 & 10,677 & 10,000,054 & 1.3\% & -\\
        \hline
        WebKB & Co-Clustering & NMI & 4,199 & 1,000 & 342,882 & 8.2\% & 4\\
        Wisconsin & Co-Clustreing & NMI & 265 & 1703 & 25,479 & 5.6\% & 5\\
        IMDB & Co-Clustering & NMI & 617 & 1878 & 20,156 & 1.7\% & 17\\
        \hline
    \end{tabular}
    \label{tab:data}
\end{table}

\paragraph{Comparison Methods.}
We compare the proposed COIN with three different groups of baseline methods: 
(1) \textbf{bipartite graph methods}: 
BiNE \cite{gao2018bine} and PinSage \cite{ying2018graph} learn node embeddings based on random walks; 
GC-MC \cite{berg2017graph} and IGMC \cite{zhang2019inductive} are matrix completion based methods; 
NeuMF \cite{he2017neural} and NGCF \cite{wang2019neural} are collaborative filtering based methods;
BiGI \cite{cao2021bipartite} is an infomax contrastive learning method.
(2) \textbf{co-clustering methods}: 
CCInfo \cite{dhillon2003information} is an information theoretic method;
SCC \cite{dhillon2001co} and SBC \cite{kluger2003spectral} are spectral methods; 
DRCC \cite{gu2009co}, CCMod \cite{ailem2015co} and SCMK \cite{kang2017twin} are matrix analysis based methods;
DeepCC \cite{xu2019deep} is a deep learning method.
(3) \textbf{graph embedding methods}:
DeepWalk \cite{perozzi2014deepwalk}, LINE \cite{tang2015line}, Node2vec \cite{grover2016node2vec} and Metapath2vec \cite{dong2017metapath2vec} are random walk based embedding methods;
VGAE \cite{kipf2016variational} learns node embeddigns by reconstructing the adjacency matrix;
HDI \cite{jing2021hdmi} maximizes high-order mutual information;
DMGI \cite{park2020unsupervised} and HDMI \cite{jing2021hdmi} are extensions of DGI \cite{velickovic2019deep} to multiplex heterogeneous graphs.
The results of baseline methods are copied from respective papers, except for GraphCL, HDI and HDMI.
For GraphCL, HDI and HDMI, we use their official implementations and set the embedding size as 128, which is the same as COIN and other baselines.

\paragraph{Bipartite Graph Encoder.}
We design a simple $L$-layer bipartite graph encoder $\mathcal{E}$ to capture 2-hop neighbor information. 
Given the embeddings of $U$ after the $(l-1)$-th layer $\mathbf{U}^{l-1}\in\mathbb{R}^{|U|\times d}$, the updating function of the $l$-th layer is:
\begin{align}
    \hat{\mathbf{V}}^{l} &= \delta_2( [\delta_1(\mathbf{A}_{VU}\mathbf{U}^{l-1}\mathbf{W}_{1}^{l})||\mathbf{V}^{l-1}]\mathbf{W}_{2}^{l})\\
    \mathbf{U}^{l} &= \delta_2( [\delta_1(\mathbf{A}_{UV}\hat{\mathbf{V}}^{l}\mathbf{W}_{3}^{l})||\mathbf{U}^{l-1}]\mathbf{W}_{4}^{l})
\end{align}
where $\mathbf{A}_{VU}\in\mathbb{R}^{|V|\times|U|}$, $\mathbf{A}_{UV}\in\mathbb{R}^{|U|\times|V|}$ are row normalized adjacency matrices;
$\mathbf{W}\in\mathbb{R}^{d\times d}$ denotes the learnable weight matrix;
$\delta$ denotes the activation function;
$[\cdot||\cdot]$ denotes the concatenation operation;
$d$ is the size of hidden dimension;
$\mathbf{U}^0$ is the randomly initialized embedding matrix.
The updating function for $V$ is the same as $U$, except for the specific weights of $\mathbf{W}$.

\paragraph{Training Details.}
For the encoder $\mathcal{E}$, we set the number of layers $L=2$, the hidden dimension $d=128$, $\delta_1$ and $\delta_2$ are LeakyReLU (negative slope is 0.1) and Tanh activation functions respectively.
The similarity function $\mathcal{S}$ is a two-layer MLP: $\mathcal{S}(\mathbf{u}, \mathbf{v})=\mathbf{W}_2(\text{Tanh}(\mathbf{W}_1[\mathbf{u}||\mathbf{v}]))$, where $\mathbf{W}_1\in\mathbb{R}^{d\times 2d}$, $\mathbf{W}_2\in\mathbb{R}^{d\times d}$, $d=128$.
For simplicity, we keep $N_K=N_L$.
For Wikipedia, the number of epochs is 50, the number of clusters $N_K=N_L=4$ and $\lambda=10$.
For ML-100K and ML-10M, the number of epochs is 100, the number of clusters $N_K=N_L=5$ and $\lambda=1$.
For the co-clustering datasets, the number of epochs is 100 and $\lambda$ is tuned within $[0.01, 0.1, 1, 10]$.
The numbers of clusters are 4, 3, 17 for WebKB, Wisconsin and IMDB.
The learning rate is fixed as 0.0005, and the optimizer is Adam \cite{kingma2014adam}.
Dropout with $p=0.5$ is applied to each layer of $\mathcal{E}$.
COIN is implemented by PyTorch \cite{paszke2019pytorch} and trained with one NVIDIA Tesla V-100 GPU.
We run COIN three times and report its mean and standard deviation.
We will release the code upon publication.

\begin{table}[t]
    \centering
    \scriptsize
    \caption{Performance (\%) of link prediction on Wikepedia}
    \begin{tabular}{ccccc}
        \hline
        \multirow{2}{*}{Method} & \multicolumn{2}{c}{Wiki (50\%)} & \multicolumn{2}{c}{Wiki (40\%)}\\
        \cmidrule(rl){2-3} \cmidrule(rl){4-5}
        & AUC-ROC & AUC-PR & AUC-ROC & AUC-PR\\
        \hline
        DeepWalk & 87.19 & 85.30 & 81.60 & 80.29\\
        LINE & 66.69 & 71.49 & 64.28 & 69.89\\
        Node2vec & 89.37 & 88.12 & 88.41 & 87.55\\
        VGAE & 87.81 & 86.93 & 86.32 & 85.74\\
        Metapath2vec & 87.20 & 84.94 & 86.75 & 84.63\\
        GraphCL & 94.40 & 94.88 & 93.67 & 94.25  \\
        HDI & 93.74 & 94.50 & 93.00 & 93.86 \\
        DMGI & 93.02 & 93.11 & 92.01 & 92.14\\
        HDMI & 94.18 & 94.86 & 93.57 & 94.23\\
        PinSage & 94.27 & 93.95 & 92.79 & 92.56\\
        BiNE & 94.33 & 93.93 & 93.15 & 93.34\\
        GC-MC & 91.90 & 92.19 & 91.40 & 91.74\\
        IGMC & 92.85 & 93.10 & 91.90 & 92.19\\
        NeuMF & 92.62 & 93.38 & 91.47 & 92.63\\
        NGCF & 94.26 & 94.07 & 93.06 & 93.37\\
        BiGI & 94.91 & 94.75 & 94.08 & 94.02\\
        \hline
        COIN & \textbf{95.30} & \textbf{95.05} & \textbf{94.53} & \textbf{94.44}\\
        \hline
    \end{tabular}
    \label{tab:wiki}
\end{table}

\subsection{Main Results}

\paragraph{Link Prediction.}
Given the learned embeddings $\mathbf{U}$, $\mathbf{V}$ and the edges $E$, we train a logistic regression classifier, which will then be evaluated on the test data.
The experimental results are presented in Table \ref{tab:wiki}. 
Comparing the bipartite graph embedding methods with the graph embedding methods, we can observe that bipartite graph methods usually have higher scores.
For random walk based methods, BiNE and PinSage have higher scores than DeepWalk and Node2vec.
For contrastive learning methods, COIN performs better than HDMI and GraphCL.
This observation indicates the necessity of introducing methods specifically for bipartite graphs.
Within the group of bipartite graph methods, recent contrastive learning methods, (COIN and BiGI) achieve much higher scores than the traditional random walk methods (BiNE and PinSage), matrix completion methods (GC-MC and IGMC), and collaborative filtering methods (NeuMF and NGCF). 
That is to say, contrastive learning is a better paradigm for learning node embeddings. 
Finally, among all the methods, COIN achieves the best scores on all metrics, indicating the power of COIN.

\paragraph{Top-K Recommendation.}
The results on ML-100K and ML-10M are presented in Tables \ref{tab:ml_100k}-\ref{tab:ml_10m}.
For ML-100K, comparing the two recent contrastive learning methods BiGI and HDMI from the bipartite graph embedding group and the graph embedding group, it can be observed that BiGI performs much better (e.g. 23.36 v.s. 20.51 on F1@10), indicating the importance of designing particular models for bipartite graphs.
The proposed COIN can further outperform BiGI on all of the metrics (e.g. 72.76 v.s. 68.78 on MRR@10), indicating the effectiveness of the proposed COIN.
For ML-10M, BiGI and HDMI are competitive to each other. 
Nevertheless, COIN significantly outperforms both BiGI and HDMI to a large margin (e.g. 21.39 v.s. 16.12 v.s. 15.37 on F1@10). 
This result again demonstrates the necessity of tailoring methods for bipartite graphs and the effectiveness of COIN.


\paragraph{Co-Clustering.}
The proposed COIN can directly produce the cluster probabilities for the given node $u$ or $v$.
We assign the cluster index with the highest probability as the cluster assignment for the given node.
The experimental results are presented in Table \ref{tab:co_clustering}.
Among all the baseline methods, the deep learning method (DeepCC) achieves the highest NMI scores on all datasets, showing the power of deep neural networks.
COIN has further improvements over DeepCC, indicating the effectiveness of the proposed COIN for discovering the clusters.

\begin{table}[t]
    \centering
    \scriptsize
    \setlength\tabcolsep{2pt} 
    \caption{Performance (\%) of top-K recommendation on ML-100K}
    \begin{tabular}{ccccccccccc}
        \hline
        Method & F1@10 & NDCG@3 & NDCG@5 & NDCG@10 & MAP@3 & MAP@5 & MAP@10 & MRR@3 & MRR@5 & MRR@10\\
        \hline
        DeepWalk \cite{perozzi2014deepwalk} & 14.20 & 7.17 & 9.32 & 13.13 & 2.72 & 3.54 & 4.92 & 43.86 & 46.83 & 48.75\\
        LINE \cite{tang2015line} & 13.71 & 6.52 & 8.57 & 12.37 & 2.45 & 3.26 & 4.67 & 44.16 & 44.37 & 46.30\\
        Node2vec \cite{grover2016node2vec} & 14.13 & 7.69 & 9.91 & 13.41 & 3.07 & 3.90 & 5.19 & 44.80 & 48.02 & 49.78\\
        VGAE \cite{kipf2016variational} & 11.38 & 6.43 & 8.18 & 10.93 & 2.35 & 2.95 & 3.94 & 39.39 & 42.32 & 43.68\\
        GraphCL \cite{you2020graph} & 19.46 & 10.13 & 13.24 & 18.17 & 4.17 & 5.65 & 8.04 & 58.04 & 60.67 & 61.97 \\
        HDI \cite{jing2021hdmi} & 19.44 & 9.76 & 12.86 & 18.01 & 4.08 & 5.56 & 8.00 & 55.82 & 58.47 & 60.01\\
        Metapath2vec \cite{dong2017metapath2vec} & 14.11 & 7.88 & 9.87 & 13.35 & 2.85 & 3.71 & 5.08 & 45.49 & 48.74 & 49.83\\
        DMGI \cite{park2020unsupervised} & 19.58 & 10.16 & 13.13 & 18.31 & 3.98 & 5.33 & 7.82 & 59.33 & 61.37 & 62.71\\
        HDMI \cite{jing2021hdmi} & 20.51 & 11.07 & 14.32 & 19.42 & 4.59 & 6.18 & 8.67 & 62.25 & 64.38 & 65.44\\
        \hline
        PinSage \cite{ying2018graph} & 21.68 & 10.95 & 14.51 & 20.27 & 4.52 & 6.18 & 9.13 & 62.56 & 64.77 & 65.76\\
        BiNE \cite{gao2018bine} & 14.83 & 7.69 & 9.96 & 13.79 & 2.87 & 3.80 & 5.24 & 48.14 & 50.94 & 52.51\\
        GC-MC \cite{berg2017graph} & 20.65 & 10.88 & 13.87 & 19.21 & 4.41 & 5.84 & 8.43 & 60.60 & 62.21 & 63.53\\
        IGMC \cite{zhang2019inductive} & 18.81 & 9.21 & 12.20 & 17.27 & 3.50 & 4.82 & 7.18 & 56.89 & 59.13 & 60.46\\
        NeuMF \cite{he2017neural} & 17.03 & 8.87 & 11.38 & 15.89 & 3.46 & 4.54 & 6.45 & 54.42 & 56.39 & 57.79\\
        NGCF \cite{wang2019neural} & 21.64 & 11.03 & 14.49 & 20.29 & 4.49 & 6.15 & 9.11 & 62.56 & 64.62 & 65.55\\
        BiGI \cite{cao2021bipartite} & 23.36 & 12.50 & 15.92 & 22.14 & 5.41 & 7.15 & 10.50 & 66.01 & 67.70 & 68.78\\
        \hline
        COIN & \textbf{24.78} & \textbf{13.48} & \textbf{17.37} & \textbf{23.62} & \textbf{5.71} & \textbf{7.82} & \textbf{11.34} & \textbf{70.58} & \textbf{72.14} & \textbf{72.76}\\
        \hline
    \end{tabular}
    \label{tab:ml_100k}
\end{table}
\begin{table}[t]
    \centering
    \scriptsize
    \setlength\tabcolsep{2pt} 
    \caption{Performance (\%) of top-K recommendation on ML-10M}
    \begin{tabular}{ccccccccccc}
        \hline
        Method & F1@10 & NDCG@3 & NDCG@5 & NDCG@10 & MAP@3 & MAP@5 & MAP@10 & MRR@3 & MRR@5 & MRR@10\\
        \hline
        DeepWalk \cite{perozzi2014deepwalk} & 7.25 & 3.12 & 4.39 & 6.50 & 1.12 & 1.65 & 2.55 & 19.14 & 20.97 & 22.45\\
        LINE \cite{tang2015line} & 6.93 & 3.07 & 4.21 & 6.24 & 1.09 & 1.55 & 2.37 & 19.69 & 21.54 & 23.08\\
        Node2vec \cite{grover2016node2vec} & 6.36 & 2.82 & 3.84 & 5.71 & 1.00 & 1.40 & 2.14 & 18.10 & 19.83 & 21.32\\
        VGAE \cite{kipf2016variational} & 11.82 & 5.00 & 6.97 & 10.61 & 1.88 & 2.79 & 4.65 & 34.75 & 37.13 & 39.00\\
        GraphCL \cite{you2020graph} & 14.88 & 7.58 & 10.00 & 14.02 & 2.99 & 4.24 & 6.51 & 47.65 & 49.55 & 50.71 \\
        HDI \cite{jing2021hdmi} & 13.51 & 7.17 & 9.27 & 12.88 & 2.85 & 3.98 & 6.05 & 45.46 & 46.93 & 48.04\\
        Metapath2vec \cite{dong2017metapath2vec} & 8.28 & 3.26 & 4.66 & 7.21 & 1.18 & 1.79 & 2.98 & 19.99 & 21.92 & 23.50\\
        DMGI \cite{park2020unsupervised} & 12.52 & 6.03 & 8.09 & 11.69 & 2.15 & 3.04 & 4.77 & 42.78 & 44.86 & 46.08\\
        HDMI \cite{jing2021hdmi} & 15.37 & 8.11 & 10.58 & 14.68 & 3.22 & 4.53 & 6.84 & 51.01 & 52.80 & 54.02 \\
        \hline
        PinSage \cite{ying2018graph} & 14.93 & 7.53 & 10.07 & 14.14 & 2.70 & 3.81 & 5.85 & 45.72 & 47.58 & 48.96\\
        GC-MC \cite{berg2017graph} & 14.74 & 7.05 & 9.42 & 13.73 & 2.58 & 3.68 & 5.88 & 48.07 & 49.95 & 51.18\\
        IGMC \cite{zhang2019inductive} & 13.68 & 6.58 & 8.70 & 12.78 & 2.41 & 3.32 & 5.22 & 45.57 & 47.82 & 49.29\\
        NeuMF \cite{he2017neural} & 13.91 & 6.58 & 8.92 & 12.93 & 2.38 & 3.41 & 5.34 & 45.82 & 48.14 & 49.57\\
        NGCF \cite{wang2019neural} & 15.11 & 7.21 & 9.67 & 14.01 & 2.67 & 3.84 & 6.16 & 48.19 & 50.15 & 51.33\\
        BiGI \cite{cao2021bipartite} & 16.12 & 7.96 & 10.41 & 15.25 & 3.02 & 4.31 & 6.77 & 49.86 & 50.66 & 51.70\\
        \hline
        COIN & \textbf{21.39} & \textbf{10.88} & \textbf{14.34} & \textbf{20.19} & \textbf{4.40} & \textbf{6.27} & \textbf{9.70} & \textbf{62.92} & \textbf{64.59} & \textbf{65.53}\\
        \hline
    \end{tabular}
    \label{tab:ml_10m}
\end{table}

\subsection{Ablation Study}
We conduct the ablation study on the Wikipedia datasets, the results of which are shown in Table \ref{tab:ablation}.
Firstly, we study the impact of the mutual information objective $I(K;L)$ by comparing the full model COIN with the version w/o $I(K;L)$. 
Evidently, the full model is significantly better than w/o $I(K;L)$, showing the importance of capturing the cluster-level information.
Secondly, we replace the MLP instantiation of the similarity function $\mathcal{S}$ in Equation \ref{eq:instance_wise_object} with commonly used cosine similarity and dot product.
It can be observed in Table \ref{tab:ablation} that there are significant performance drops, which suggests that the manifolds of the embeddings are complex and MLP is a better way for calculating the similarity between nodes than simple cosine similarity or dot product.
In most of the prior works, to improve the performance, cosine similarity and dot product are normalized by a temperature parameter $\tau$. 
However, tuning the extra parameter $\tau$ is very time-consuming, which brings extra burdens for experiments.
Thirdly, we study the influence of the prior distribution $p(U, V)$. 
If we assume the random variables $U$ and $V$ are independent $p(U, V)=p(U)p(V)$, then there will be significant performance drops, as shown in Table \ref{tab:ablation}.
This observation indicates that the learning ability of COIN is bounded by $p(U, V)$, which empirically supports Theorem \ref{theorem:diff_mi}.

\begin{table}[h!]
    \centering
    \scriptsize
    \setlength\tabcolsep{3.5pt} 
    \caption{Performance (\%) of NMI on co-clustering datasets}
    \begin{tabular}{cccccccccc}
        \hline
        Dataset & K-Means & SCC \cite{dhillon2001co} & SBC \cite{kluger2003spectral} & CCMod \cite{ailem2015co} & DRCC \cite{gu2009co} & CCInfo \cite{dhillon2003information} & SCMK \cite{kang2017twin} & DeepCC \cite{xu2019deep} & COIN\\
        \hline
        WebKB & 26.1 & 31.1 & 13.0 & 40.1 & 31.9 & 39.7 & 10.0 & 40.5 & \textbf{43.0}\\
        Wisconsin & 37.5 & 35.4 & 38.2 & 35.1 & 20.4 & 39.3 & 42.9 & 46.7 & \textbf{49.3}\\
        IMDB & 13.9 & 25.5 & 20.6 & 21.6 & 6.9 & 18.7 & 18.4 & 26.8 & \textbf{31.9}\\
        \hline
    \end{tabular}
    \label{tab:co_clustering}
    \vspace{-10pt}
\end{table}

\begin{table}[t]
    \centering
    \scriptsize
    \caption{Ablation study on Wikepedia}
    \begin{tabular}{ccccc}
        \hline
        \multirow{2}{*}{Method} & \multicolumn{2}{c}{Wiki (50\%)} & \multicolumn{2}{c}{Wiki (40\%)}\\
        \cmidrule(rl){2-3} \cmidrule(rl){4-5}
        & AUC-ROC & AUC-PR & AUC-ROC & AUC-PR\\
        \hline
        COIN & \textbf{95.30} & \textbf{95.05} & \textbf{94.53} & \textbf{94.44}\\
        \hline
        w/o $I(K;L)$ & 94.73 & 94.62 & 93.86 & 93.76\\
        $\mathcal{S}$=cosine similarity & 94.84 & 93.98 & 94.12 & 92.99\\
        $\mathcal{S}$=dot product & 94.76 & 94.56 & 94.08 & 93.76\\
        $p(U, V)=p(U)p(V)$ & 94.66 & 94.55 & 93.96 & 93.81\\
        \hline
    \end{tabular}
    \label{tab:ablation}
\end{table}

\subsection{Sensitivity Experiments}
We study the impact of the number of clusters on the WebKB and Wisconsin datasets, and study the coefficient $\lambda$ for mutual information on the Wikipedia datasets.
As shown in Figure \ref{fig:web4_k}-\ref{fig:wisconsin_k}, COIN performs the best when the number of clusters is around the ground truth number of classes (4 for WebKB and 5 for Wisconsin).
As shown in Figure \ref{fig:wiki_5_mi_record}-\ref{fig:wiki_5_mi}, the performance of COIN only has small variances for different $\lambda$. 
COIN achieves the best performance when $\lambda=10$.

\subsection{Convergence of Mutual Information}
We present the mutual information of co-clusters $I(K;L)$ for each training epoch on Wiki (50\%), Wiki (40\%), ML-100K, and ML-10M in Figure \ref{fig:mi_record}.
$I(K;L)$ increases rapidly in the first 10 epochs for all the datasets, after which $I(K;L)$ moves slowly towards the limits.
The results show that COIN has a good convergence ability.

\begin{figure}[t]
\centering
\begin{subfigure}[b]{.24\textwidth}
\includegraphics[width=\linewidth]{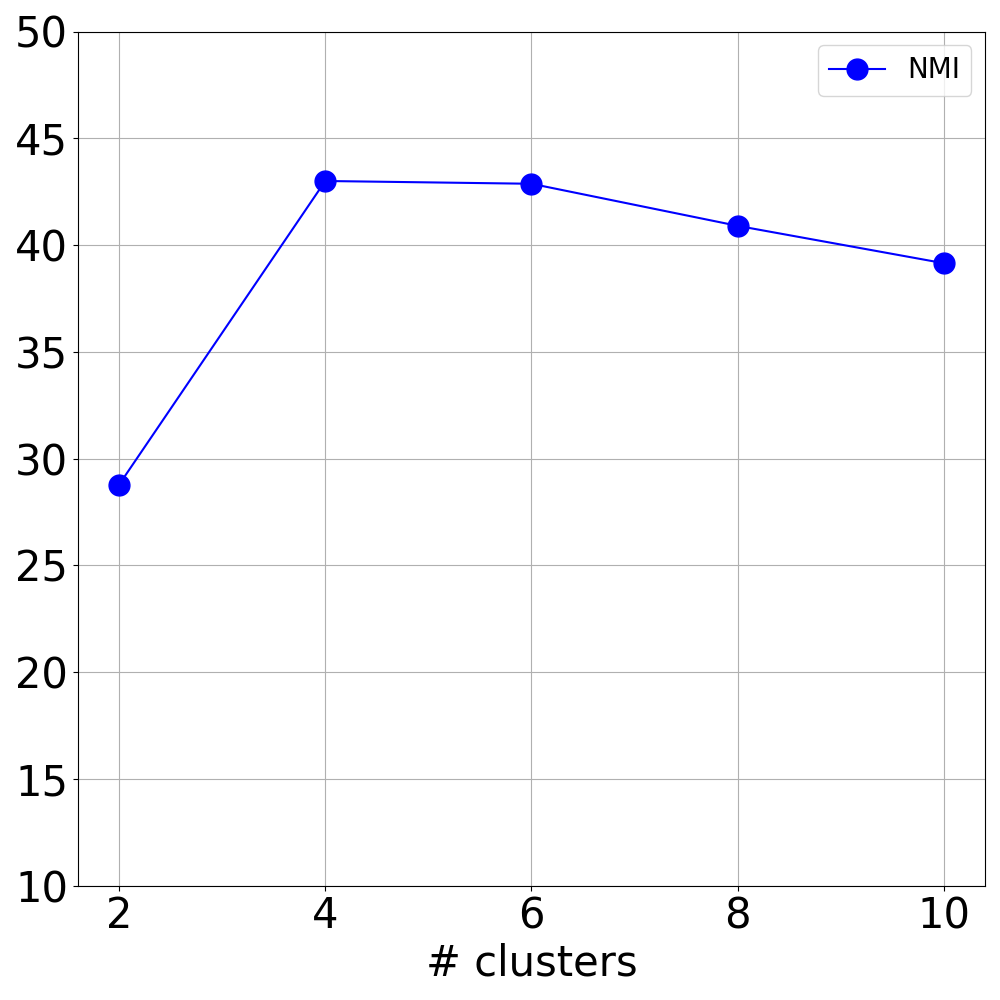}
 \caption{WebKB}\label{fig:web4_k}
\end{subfigure}\,
\begin{subfigure}[b]{.24\textwidth}
  \includegraphics[width=\linewidth]{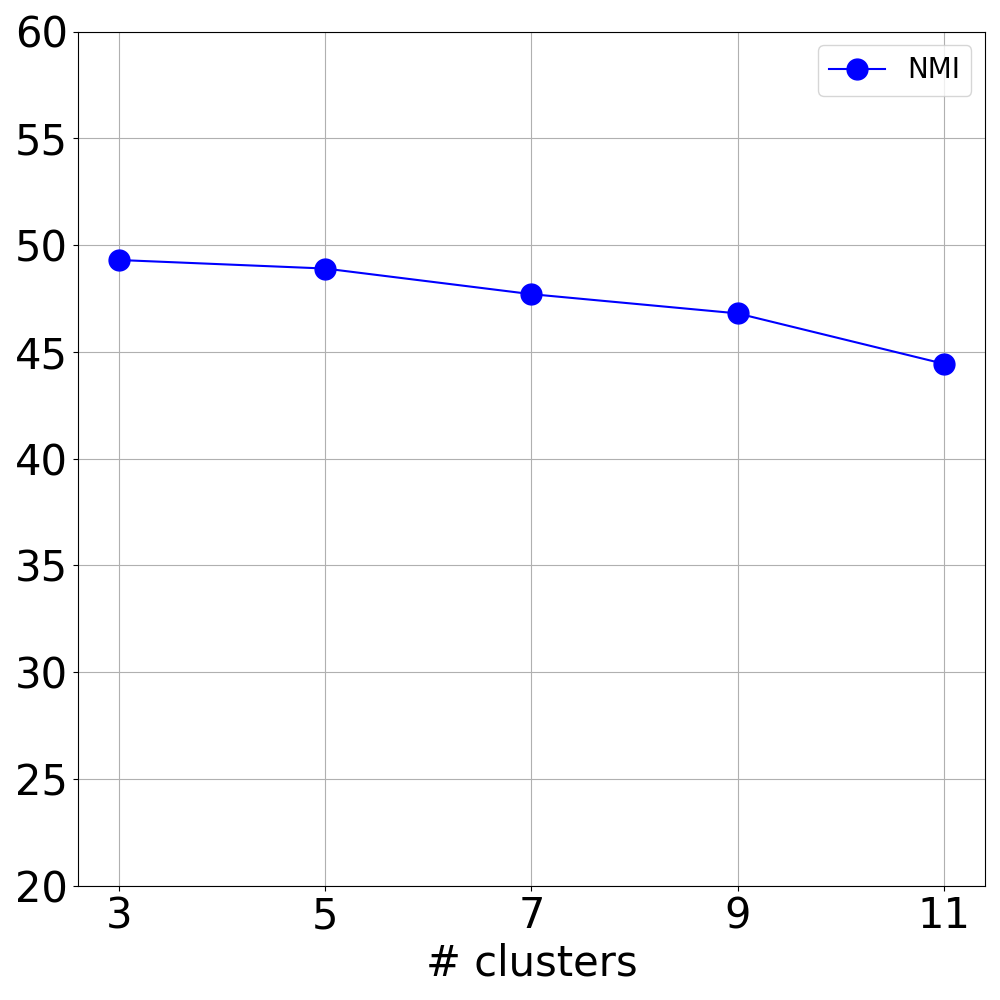}
  \caption{Wisconsin}\label{fig:wisconsin_k}
\end{subfigure}\,
\begin{subfigure}[b]{.24\textwidth}
  \includegraphics[width=\linewidth]{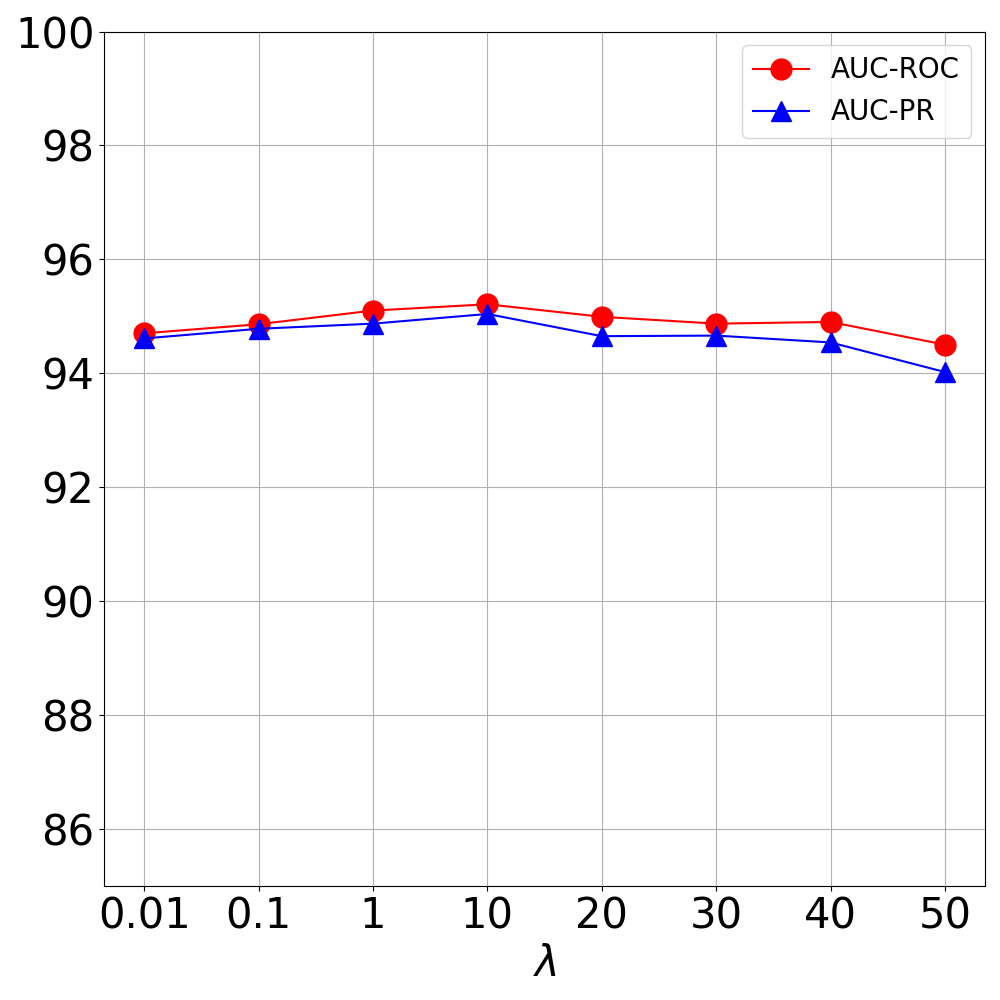}
  \caption{Wiki (50\%)}\label{fig:wiki_5_mi}
\end{subfigure}\,
\begin{subfigure}[b]{.24\textwidth}
  \includegraphics[width=\linewidth]{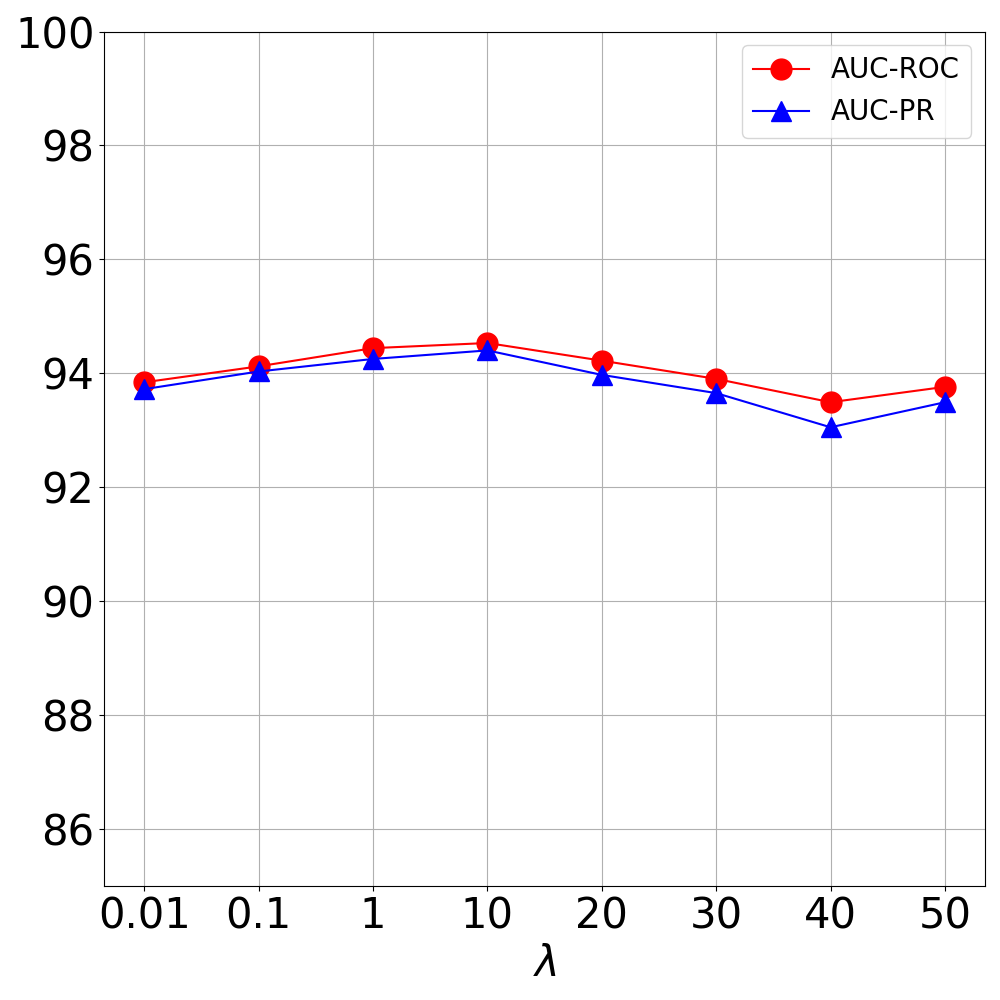}
  \caption{Wiki (40\%)}\label{fig:wiki_4_mi}
\end{subfigure}
\caption{Sensitivity experiments for (a)(b) cluster numbers and (c)(d) $\lambda$ of mutual information loss.}
\label{fig:sensitivity}
\end{figure}

\begin{figure}
\centering
\begin{subfigure}[b]{.24\textwidth}
\includegraphics[width=\linewidth]{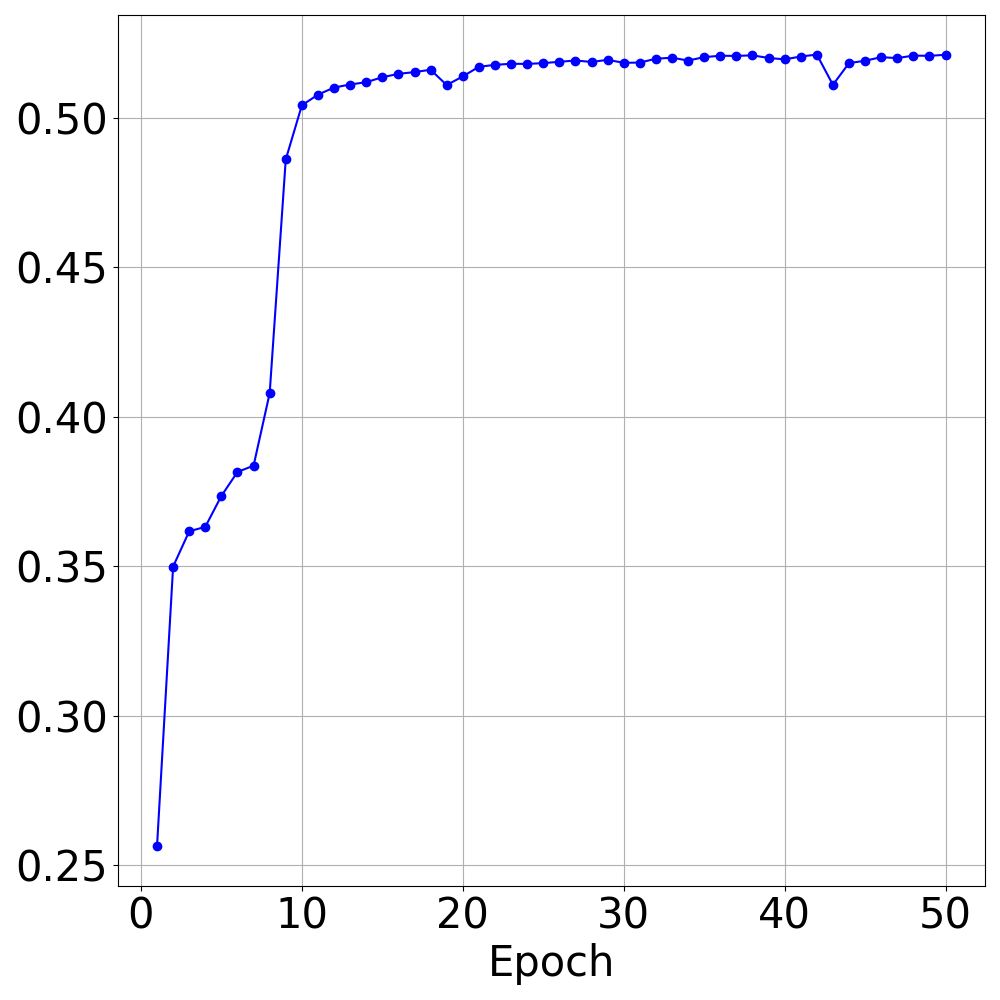}
 \caption{Wiki (50\%)}\label{fig:wiki_5_mi_record}
\end{subfigure}\,
\begin{subfigure}[b]{.24\textwidth}
  \includegraphics[width=\linewidth]{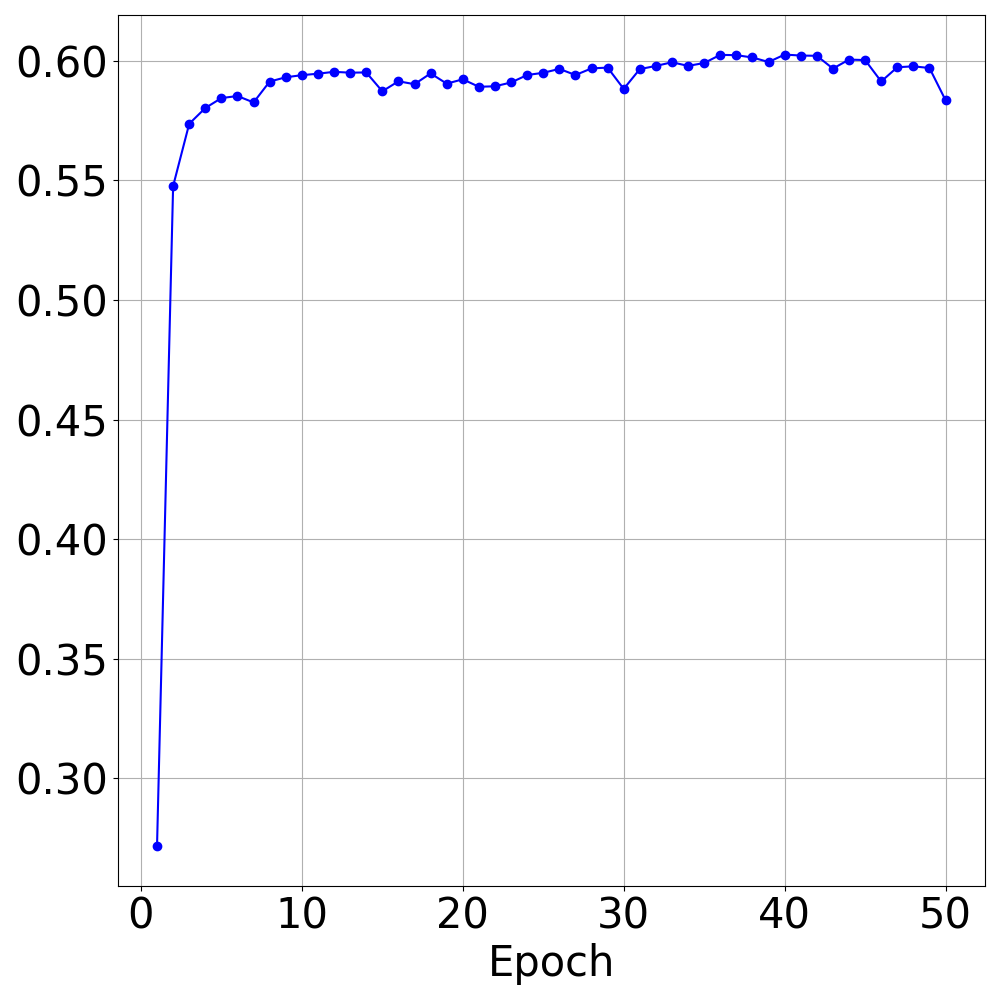}
  \caption{Wiki (40\%)}\label{fig:wiki_4_mi_record}
\end{subfigure}\,
\begin{subfigure}[b]{.24\textwidth}
  \includegraphics[width=\linewidth]{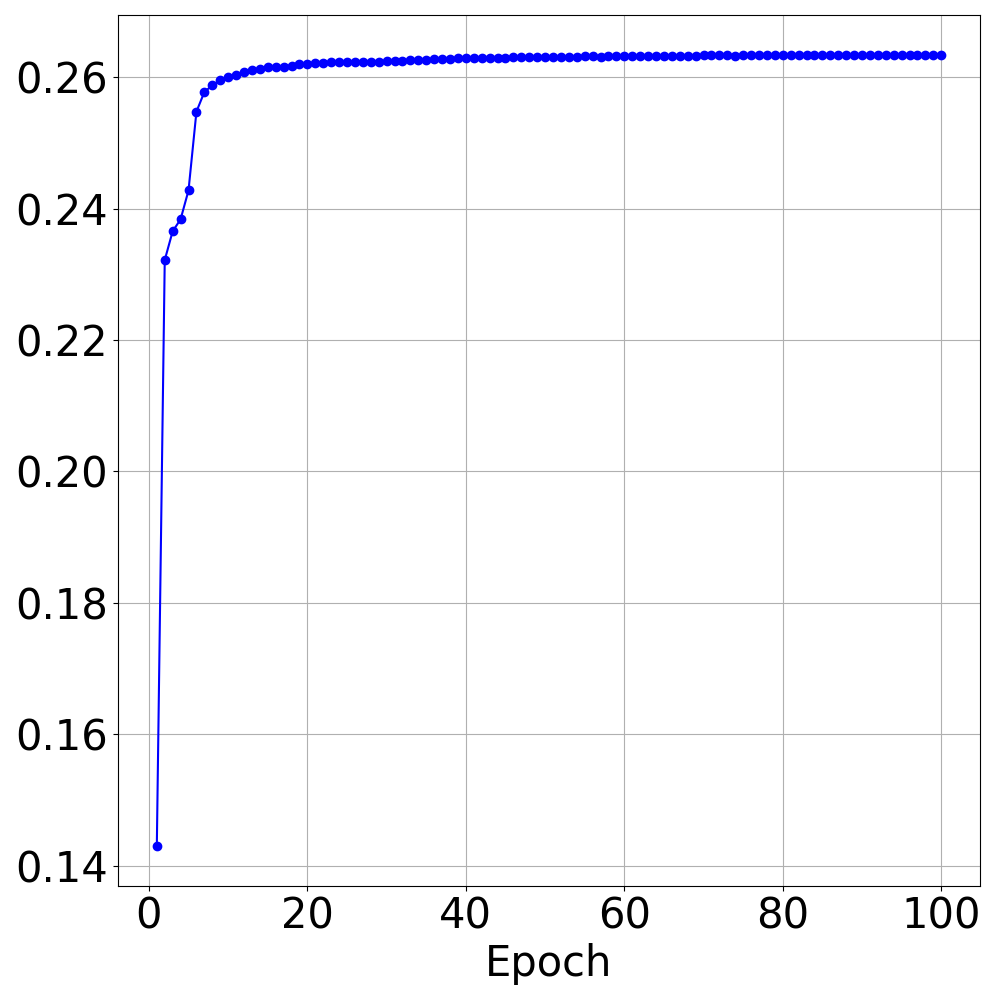}
  \caption{ML-100K}\label{fig:ml100k_mi_record}
\end{subfigure}\,
\begin{subfigure}[b]{.24\textwidth}
  \includegraphics[width=\linewidth]{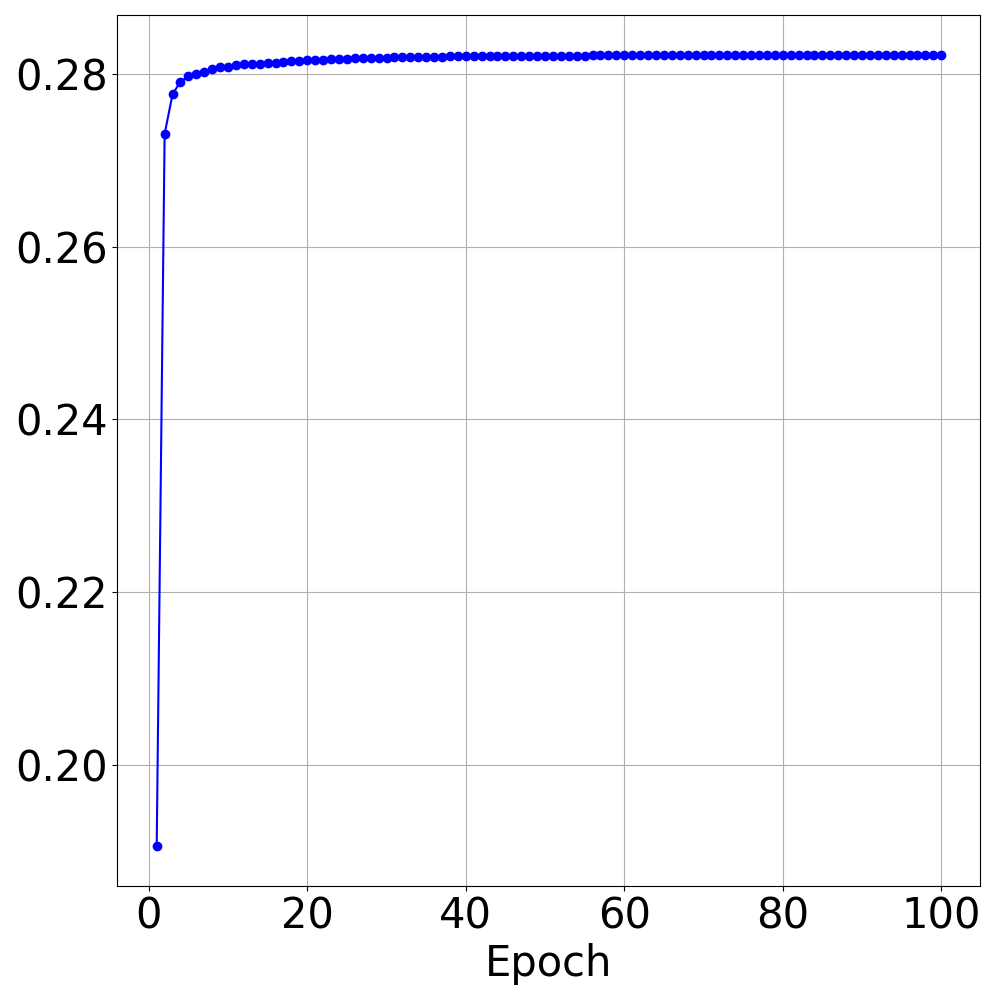}
  \caption{ML-10M}\label{fig:ml10m_mi_record}
\end{subfigure}
\caption{Mutual information of co-clusters v.s. the number of epochs}
\label{fig:mi_record}
\end{figure}

\section{Related Work}\label{sec:related_work}
\pdfoutput=1


\paragraph{Graph Embedding.}
A fundamental challenge for graph learning is to extract informative node embeddings \cite{wu2021self, perozzi2014deepwalk, velickovic2019deep, zhu2021graph, park2020unsupervised, yan2021bright}.
DeepWalk \cite{perozzi2014deepwalk} and node2vec \cite{grover2016node2vec} use the random walk to sample node pairs and maximizes the similarities for random walk neighbors.
LINE \cite{tang2015line} and SDNE  \cite{wang2016structural} consider both of the first-order and second-order proximity information when learning node embeddings.
VGAE \cite{kipf2016variational} and ARVGA \cite{pan2018adversarially} are graph auto-encoder methods which learn node embeddings by reconstructing the adjacency matrices.
Recent self-supervised graph learning methods extract node embeddings by optimizing contrastive objectives. 
DGI \cite{velickovic2019deep} maximizes the mutual information of local and global representation of graphs.
MVGRL \cite{hassani2020contrastive} maximizes mutual information of different views.
GMI \cite{peng2020graph} introduces the graphical mutual information.
GraphCL \cite{you2020graph} generates positive and negative node pairs via various augmentations.
GCA \cite{zhu2021graph} generates data augmentation adaptively according to pre-defined probabilities.
HDI \cite{jing2021hdmi} introduces a high-order mutual inforamtion objective.
All of the above methods are designed for homogeneous graphs.
Metapath2vec \cite{dong2017metapath2vec} is a random walk based method, which extends DeepWalk and node2vec to heterogeneous graphs.
Recently, DMGI \cite{park2020unsupervised} and HDMI \cite{jing2021hdmi} extends DGI and HDI to multiplex heterogeneous graphs.
Although these methods have achieved impressive scores on downstream tasks, they are not tailored for bipartite graphs and thus usually have sub-optimal performance compared with bipartite graph embedding methods. 

\paragraph{Bipartite Graph Embedding.}
Bipartite graphs have been widely used to model interactions between two disjoint sets of nodes, such as the user-item interaction in recommender systems.
Inspired by random walk based graph embeddings methods \cite{perozzi2014deepwalk, grover2016node2vec}, BiNE \cite{gao2018bine} uses both explicit connections between nodes and random walks extracted implicit relationships between nodes to learn node embeddings. 
IGE \cite{zhang2017learning} learns node embeddings based on the direct connection between nodes and the attributes of the edges.
GC-MC \cite{berg2017graph} trains a graph convolutional network and learns node embeddings by recovering the link between two nodes.
PinSage \cite{ying2018graph} is a web-scale method which combines graph convolutional network with random walks.
NeuMF \cite{he2017neural} is a neural network based collaborative filtering method.
NGCF \cite{wang2019neural} incorporates the high-order collaborative signals to improve the quality of the node embeddings.
Recently, contrastive learning has been applied to learn bipartite graph embeddings.
EGNL \cite{yang2021enhanced} is a collaborative filtering method, which maximizes the mutual information of local and global represents of the learned enhanced user-item graph.
BiGI \cite{cao2021bipartite} uses triplet loss to learn the local connectivity of nodes, and captures the global information of the graph by maximizing the mutual information of local and global representations.
Different from BiGI, the proposed COIN further captures cluster-level information.

\paragraph{Co-Clustering.}
Co-clustering aims to partition rows and columns of a co-occurrence matrix into co-clusters simultaneously \cite{dhillon2001co, dhillon2003information}.
In practice, co-clustering algorithms usually have impressive improvements over traditional one-way clustering algorithms \cite{xu2019deep}.
InfoCC \cite{dhillon2003information} is an information-theoretic method, which co-clusters document-word interaction matrices by a mutual information based objective.
SCC \cite{dhillon2001co} and SBC \cite{kluger2003spectral} are based on spectral analysis.
DRCC \cite{gu2009co} is a semi-nonnegative matrix tri-factorization method with geometric regularization.
BCC \cite{shan2008bayesian}, LDCC \cite{shafiei2006latent} and MPCCE \cite{wang2011nonparametric} are Bayesian approaches.
SOBG \cite{nie2017learning} performs co-clustering by learning a new graph similarity matrix.
CCMod \cite{ailem2015co} obtains co-clusters by maximizing graph modularity.
SCMK \cite{kang2017twin} is a kernel based method.
DeepCC \cite{xu2019deep} is the first deep learning based co-clustering method, which uses auto-encoders for dimension reduction and a variant of the Gaussian mixture model to obtain cluster assignments.
Different from DeepCC, which is specifically designed for the co-clustering task, COIN is a self-supervised learning method, which is able to deal with various downstream tasks.

\section{Conclusion}\label{sec:conclusion}
\pdfoutput=1

In this paper, we introduce a novel \underline{co}-cluster \underline{in}fomax (COIN) framework for self-supervised learning on bipartite graphs.
COIN is able to capture the cluster-level information of bipartite graphs by maximizing the mutual information of co-clusters.
There are two advantages of COIN.
On the one hand, COIN calculates the mutual information rather than estimates the mutual information as in prior works.
On the other hand, COIN is an end-to-end clustering framework which can be trained and optimized with other differentiable objectives.
Furthermore, we also provide the theoretical analysis of COIN.
We theoretically prove that COIN can maximize the mutual information of embeddings of the two types of nodes, and COIN is upper-bounded by the prior distribution assumed over the bipartite graph.
Finally, extensive empirical evaluation on diverse benchmark datasets and downstream tasks demonstrates the effectiveness of the proposed COIN.


\bibliographystyle{plain}
\bibliography{ref}

\newpage
\appendix
\pdfoutput=1

\section{Theoretical Analysis}\label{appendix:theory}

\begin{lem}\label{lemma_appendix}
For COIN, the following inequality holds:
\begin{equation}
    \log p(k)p(l) \geq \sum_{\mathbf{u}, \mathbf{v}}p(\mathbf{u}, \mathbf{v}|k, l)\log\frac{p(\mathbf{u}, k)p(\mathbf{v}, l)}{p(\mathbf{u}, \mathbf{v}|k, l)}
\end{equation}
where $k\in\{1, \cdots, N_K\}$, $l\in\{1, \cdots, N_L\}$, $\mathbf{u}\in\mathbf{U}$, $\mathbf{v}\in\mathbf{V}$.
\end{lem}

\begin{proof}
We have $p(k)=\sum_{\mathbf{u}}p(\mathbf{u}, k)$, $p(l)=\sum_{\mathbf{v}}p(\mathbf{v}, l)$, and thus $p(k)p(l)=\sum_{\mathbf{u}, \mathbf{v}}p(\mathbf{u}, k)p(\mathbf{v}, l)$.
As a result, we have:
\begin{equation}
    \log p(k)p(l) = \log\sum_{\mathbf{u}, \mathbf{v}}p(k , \mathbf{u})p(l, \mathbf{v})\frac{p(\mathbf{u}, \mathbf{v}|k, l)}{p(\mathbf{u}, \mathbf{v}|k, l)} \geq \sum_{\mathbf{u}, \mathbf{v}}p(\mathbf{u}, \mathbf{v}|k, l)\log\frac{p(\mathbf{u}, k)p(\mathbf{v}, l)}{p(\mathbf{u}, \mathbf{v}|k, l)}
\end{equation}
where the inequality holds according to Jensen's inequality.
\end{proof}

\begin{thm}[Variational Bound]\label{theorem_appendix:variational_bound}
The mutual information $I(\mathbf{U};\mathbf{V})$ of embeddings $\mathbf{U}$ and $\mathbf{V}$ is lower-bounded by the mutual information of co-clusters $I(K;L)$:
\begin{equation}\label{eq_appendix:theorem_bound}
    I(K;L)\leq I(\mathbf{U};\mathbf{V})
\end{equation}
\end{thm}

\begin{proof}
According to the definition of mutual information and Lemma \ref{lemma_appendix}, we have:
\begin{equation}\label{eq_appendix:thorem_1_1}
    I(K; L)= \sum_{k, l}p(k, l)\frac{\log p(k, l)}{\log p(k)p(l)} \leq \sum_{k, l}p(k, l)\sum_{\mathbf{u}, \mathbf{v}}p(\mathbf{u}, \mathbf{v}|k, l)\log \frac{p(k, l)p(\mathbf{u}, \mathbf{v}|k, l)}{p(\mathbf{u}, k)p(\mathbf{v}, l)} = R
\end{equation}
By merging $p(k, l)$ with $p(\mathbf{u}, \mathbf{v}|k, l)$, we have:
\begin{equation}
    R = \sum_{\mathbf{u}, \mathbf{v}, k, l}p(\mathbf{u}, \mathbf{v}, k, l)\log \frac{p(\mathbf{u}, \mathbf{v}, k, l)}{p(\mathbf{u}, k)p(\mathbf{v}, l)} = \sum_{\mathbf{u}, \mathbf{v}, k, l}p(\mathbf{u}, \mathbf{v}, k, l)\log \frac{p(k, l|\mathbf{u}, \mathbf{v})p(\mathbf{u}, \mathbf{v})}{p(\mathbf{u})p(k|\mathbf{u})p(\mathbf{v})p(l|\mathbf{v})}
\end{equation}
Since cluster networks $\mathcal{C}_U$ and $\mathcal{C}_V$ have separate sets of parameters and inputs, therefore, it is natural to have $p(k, l|\mathbf{u}, \mathbf{v})=p(k|\mathbf{u})p(l|\mathbf{v})$.
As a result, we have:
\begin{equation}
    R = \sum_{\mathbf{u}, \mathbf{v}, k, l}p(\mathbf{u}, \mathbf{v}, k, l)\log \frac{p(\mathbf{u}, \mathbf{v})}{p(\mathbf{u})p(\mathbf{v})} = \sum_{\mathbf{u}, \mathbf{v}}p(\mathbf{u}, \mathbf{v})\log \frac{p(\mathbf{u}, \mathbf{v})}{p(\mathbf{u})p(\mathbf{v})} = I(\mathbf{U};\mathbf{V})
\end{equation}
Take the above result back to Inequality \eqref{eq_appendix:thorem_1_1}, and we will obtain Inequality \eqref{eq_appendix:theorem_bound}. 
\end{proof}

\begin{thm}[Mutual Information Difference]\label{theorem_appendix:diff_mi}
Given a soft co-clustering function $\phi=(\phi_U, \phi_V)$, where $\phi_U$ and $\phi_V$ are deterministic functions mapping nodes $u$,$v$ to their cluster distributions $p(K|u)$ and $p(L|v)$, and the prior joint distribution $p(U, V)$ assumed on $U$, $V$, we have:
\begin{equation}
    I(U, V) - I(K;L) = D_{KL}(p(K,L,U, V)||q(K,L,U, V))
\end{equation}
where $D_{KL}$ denotes the KL-divergence, and $q(K,L,U,V)=p(K,L)p(U|K)p(V|L)$.
\end{thm}
\begin{proof}
According to the definition of mutual information, we have:
\begin{align}
    I(U;V) - I(K;L) &= \sum_{k, l, u, v}p(k, l, u, v)\log\frac{p(u, v)}{p(u)p(v)} - \sum_{k, l, u, v}p(k, l, u, v)\log\frac{p(k, l)}{p(k)p(l)}\\
    &=\sum_{k, l, u, v}p(k, l, u, v)\log\frac{p(u, v)p(k)p(l)}{p(u)p(v)p(k, l)}\label{eq:log_appendix}
\end{align}
Denoting the argument of the logarithm term as $\Delta$, we have:
\begin{equation}
    \Delta = \Delta\cdot\frac{p(k, l|u, v)}{p(k, l|u, v)} = \frac{p(k)p(l)}{p(u)p(v)p(k, l)}\cdot\frac{p(k, l, u, v)}{p(k|u)p(l|v)} = \frac{p(k, l, u, v)}{p(k, l)p(u|k)p(v|l)}
\end{equation}
In the second equality, $p(k, l|u, v)=p(k|u)p(l|v)$ comes from the fact that $\phi_U$ takes $U$ as the input and produces $p(K|U)$, and thus given $U$, $K$ is conditional independent to $V$ and $L$. Similarly, given $V$, $L$ is conditional independent to $U$ and $K$.
In the third equality, Bayes's rule is applied.

Take $\Delta$ back to Equation \eqref{eq:log_appendix}, we have:
\begin{equation}
    I(U;V) - I(K;L) = \sum_{k, l, u, v}p(k, l, u, v)\log\frac{p(k, l, u, v)}{p(k, l)p(u|k)p(v|l)} = D_{KL}(p||q)
\end{equation}
where $p=p(K,L,U,V)$ and $q=p(K,L)p(U|K)p(V|L)$. 
It is obvious that $\sum_{k, l, u, v}q=1$.
\end{proof}

\begin{table}[t]
    \centering
    \scriptsize
    \caption{Descriptions of datasets}
    \begin{tabular}{cccccccc}
        \hline
        Dataset & Task & Evaluation & $|U|$ & $|V|$ & $|E|$ & Density & \# Class\\
        \hline
        Wikipedia & Link Prediction & AUC-ROC, AUC-PR & 15,000 & 3,214 & 64,095 & 0.1\% & -\\
        \hline
        ML-100K & Top-K Recommendation & F1, NDCG, MAP, MRR & 943 & 1,682 & 100,000 & 6.3\% & -\\
        ML-10M & Top-K Recommendation & F1, NDCG, MAP, MRR & 69,878 & 10,677 & 10,000,054 & 1.3\% & -\\
        \hline
        WebKB & Co-Clustering & NMI & 4,199 & 1,000 & 342,882 & 8.2\% & 4\\
        Wisconsin & Co-Clustreing & NMI & 265 & 1703 & 25,479 & 5.6\% & 5\\
        IMDB & Co-Clustering & NMI & 617 & 1878 & 20,156 & 1.7\% & 17\\
        \hline
    \end{tabular}
    \label{tab:data_appendix}
\end{table}

\section{Datasets}\label{appendix:data}
We evaluate the proposed COIN on six public benchmark datasets with three different tasks.
The descriptions of the datasets are presented in Table \ref{tab:data_appendix}.

\paragraph{Wikipedia}
The Wikipedia dataset\footnote{\url{https://github.com/clhchtcjj/BiNE/tree/master/data/wiki}} contains the edit relationship between authors and pages, which is used for link prediction \cite{gao2018bine, cao2021bipartite}.
We use the data processed by \cite{cao2021bipartite}, which has two different splits 50\% and 40\% for training.

\paragraph{ML-100K}
This dataset\footnote{\url{https://grouplens.org/datasets/movielens/100k/}} \cite{harper2015movielens} is collected through the MovieLens \footnote{\url{https://movielens.org/}} website, which contains 100 thousand movie ratings from 943 users on 1682 movies. 
Each user has rated at least 20 movies, and the relation between users and items are binary.
We use the data processed by \cite{cao2021bipartite}.
This data is used for top-K recommendation.

\paragraph{ML-10M}
This dataset\footnote{\url{https://grouplens.org/datasets/movielens/10m/}}\cite{harper2015movielens} is collected through the MovieLens \footnote{\url{https://movielens.org/}} website, which contains 10 million movie ratings.
All users selected had rated at least 20 movies, and the relation between users and items are binary.
We use the data processed by \cite{cao2021bipartite}.
This data is used for top-K recommendation.

\paragraph{WebKB}
The WebKB dataset is about the information of the web pages, which is formulated as a document-keyword interaction bipartite graph. 
It contains 4 different classes. 
This dataset is processed by \cite{xu2019deep}, and it is used for co-clustering.

\paragraph{Wisconsin}
The Wisconsin dataset contains 265 documents with 5 classes, i.e. student, project, course, staff and faculty.
We use the version processed by \cite{xu2019deep}, and it is used for co-clustering.

\paragraph{IMDB}
The IMDB dataset is a document-keyword interaction bipartite graph, where documents are descriptions for movies.
Following \cite{xu2019deep}, we use it for co-clustering.

\end{document}